\documentclass[11pt]{amsart}

\usepackage{graphicx,epstopdf}
\usepackage{amsmath,amsfonts,amssymb,amsthm}
\usepackage[colorlinks,citecolor=blue,urlcolor=blue]{hyperref}
\usepackage{enumerate}
\usepackage{tikz}
\usetikzlibrary{positioning}
\usepackage{booktabs,algorithm2e}

\usepackage{multirow, float} % We can delete this after removing the table with condition numbers
\newtheorem{thm}{Theorem}[section]

\newtheorem{prop}[thm]{Proposition}
\newtheorem{cor}[thm]{Corollary}

\theoremstyle{definition}
\newtheorem{defn}[thm]{Definition}

\newtheorem{ex}[thm]{Example}

\theoremstyle{remark}
\newtheorem{rem}[thm]{Remark}

\newcommand{\cA}{\mathcal{A}}

\newcommand{\E}{\mathbb{E}}

\newcommand{\cM}{\mathcal{M}}

\newcommand{\N}{\mathbb{N}}

\renewcommand{\P}{\mathbb{P}}

\newcommand{\R}{\mathbb{R}}

\newcommand{\cX}{\mathcal{X}}

\newcommand{\cY}{\mathcal{Y}}
\newcommand{\indep}{\mbox{\,$\perp\!\!\!\perp$\,}}

\title[]{Robust Estimation of Tree Structured Models}
\author[]{Marta Casanellas}
\address{Institut de Matem\`{a}tiques de la UPC-BarcelonaTech (IMTech), Universitat Polit\`{e}cnica de Catalunya, Centre de Recerca Matemàtica, Barcelona, Spain}
\email{marta.casanellas@upc.edu}
\author[]{Marina Garrote-L\'{o}pez}
\address{Department of Mathematics, Universitat Polit\`{e}cnica de Catalunya, Barcelona, Spain}
\email{marina.garrote@upc.edu}
\author[]{Piotr Zwiernik}
\address{Department of Economics and Business, Universitat Pompeu Fabra, Barcelona, Spain}
\email{piotr.zwiernik@upf.edu}
\subjclass[2020]{}
\date{\today}
\begin{document}
\begin{abstract}
Consider the problem of learning undirected graphical models on trees  from corrupted data. Recently \cite{katiyar2020robust} showed that it is possible  to recover trees from noisy binary data up to a small equivalence class of possible trees. Their other paper on the Gaussian case follows a similar pattern. By framing this as a special phylogenetic recovery problem we largely generalize these two settings. Using the framework of linear latent tree models we discuss tree identifiability for binary data under a continuous corruption model. For the Ising and the Gaussian tree model we also provide a characterisation of when the Chow-Liu algorithm consistently learns the underlying tree from the noisy data.
  \end{abstract}

\maketitle
%\tableofcontents

\section{Introduction}

Probabilistic graphical models form a popular family of statistical models used to describe dependence structure in multivariate scenarios. A particularly simple instance of a graphical model is when the  underlying graph is a tree. Despite its simplicity, these models can be useful in image/video classification, for exploratory analysis in high-dimensional settings, and as first approximations in more complicated systems; see \cite{badrinarayanan2013semi,bresler2020learning,d2003tree,katiyar2019robust,katiyar2020robust,nikolakakis2018predictive} and references therein.

Parameter estimation, inference, and structure learning is particularly easy in the case of tree models. Given a random sample from a tree distribution, there is an efficient way of finding the maximum likelihood tree given by Chow and Liu \cite{chow1968approximating}. Chow and Liu showed that finding the maximum likelihood tree can be formulated as a maximum weight spanning tree problem based on mutual informations --- a task for which highly efficient algorithms  exist.

Following \cite{katiyar2020robust} we consider the case when the observed random sample is a corrupted version of the original random sample. The aim of this paper is to study the most general situation in which recovery of  the true tree is possible. Our first main result is a generalization of the identifiability result in \cite{katiyar2020robust} to the situation of arbitrary discrete variables with an arbitrary but equal number of states. This is Theorem~\ref{th:main}, where we prove that the original tree can be recovered from a noisy distribution, up to label swapping of certain nodes. Then we specify (mild) conditions on the noise which guarantee that the complete original tree is identifiable from a  noisy distribution (see Theorem \ref{th:identifyexactly}).

Our approach relies on the observation that the distribution of the corrupted data lies in a latent tree model. Then standard identifiability results for phylogenetic models can be employed \cite{chang1996full,semple2003phylogenetics}. A similar observation has been applied for the binary data case in \cite{nikolakakis2018predictive}, where high-probability sample complexity guarantees for exact structure recovery were provided.

This insight allows to study the tree recovery problem in much detail. For example, there is some debate on whether the Chow-Liu algorithm is applicable in the case of noisy data. Some sufficient conditions have been studied in \cite{nikolakakis2018predictive}. Here, in Theorem~\ref{th:CL}, we give sufficient and necessary conditions in the case of the binary data and with a very similar argument in the Gaussian case.

To generalize from the standard discrete setting we discuss the linear latent tree models \cite{anandkumar2011spectral,zwiernik2018latent}. This allows to generalize our main result to the (multivariate) Gaussian case and beyond; see Theorem~\ref{th:mainlinear}. In particular, in Section~\ref{sec:conntinuous} we present a simple model for discrete data with a continuous noise model and we illustrate with simulations how tree recovery performs.

The paper is organized as follows. In Section~\ref{sec:problem} we define tree model and the noisy tree distributions. In Section~\ref{sec:phylo} we define latent tree models, we recall classical results on structure identifiability, and we show how this problem links to the original problem of recovering the underlying tree from noisy data. The main results related to this problem are stated in Section~\ref{sec:identify}. In Section~\ref{sec:treelearning} we further build upon these results by studying consistency of the Chow-Liu algorithm and by providing numerical examples of how standard phylogenetic recovery methods perform in recovering the true tree $T^*$. The results in Section~\ref{sec:identify} are further generalized to linear tree models in Section~\ref{sec:linear}.

There are many types trees that appear in this paper. For reader's convenience we summarize our notation:
\begin{enumerate}
	\item [$T$] a general tree,
	\item [$T^*$] the true tree in the underlying tree distribution,
	\item [$T^e$] the tree obtained from $T^*$ by adding a copy of each vertex and linking it to its counterpart in $T^*$.
	\item [$\overline T$] the tree obtained from $T$ by suppressing all the degree two nodes (see Definition~\ref{def:contract}); $\overline{T^e}$ is a special case of this notation with $T=T^e$.
\end{enumerate}

\section{Problem formulation}\label{sec:problem}

In this section we set-up our problem in the case of discrete data. This will be extended later in Section~\ref{sec:linear} to linear models on trees with the Gaussian model as a special case.

\subsection{Tree distributions}\label{sec:treeMarkov}
Let $X=(X_1,\ldots,X_d)$  be a random vector with values in a finite product space $\cX=\prod_{i=1}^d \cX_i$. Without loss of generality we assume $\cX_{i}=\{0,\ldots,r_i-1\}$, $r_i\in \N$, $r_i\geq 2$. Let $T^*$ be  a tree with vertices $V=\{1,\ldots,d\}$, representing the components of the random vector $X=(X_1,\ldots,X_d)$, and with edges $E^*$. The distribution $p$ of $X$ is \emph{Markov with respect to} $T^*$ if
$$
p(x)\;=\;\prod_{ij\in E^*} \phi_{ij}(x_i,x_j)\qquad\mbox{for all } x\in \cX,
$$
where $\phi_{ij}:\cX_i\times \cX_j\to (0,+\infty)$ are some functions, called potentials. By the Hammersley-Clifford theorem we get then conditional independence characterization in terms of separation  in the tree,  that is, $X_i\indep X_j|X_C$ if $C\subset V$ separates $i$ and $j$ in $T^*$; c.f. Theorem~3.9 in \cite{lau96}. In the binary  case, when $\cX_i=\{0,1\}$, we equivalently write
$$
p(x)\;=\;\frac{1}{Z(h,\beta)}\exp\Big\{\sum_{i\in V}  h_i x_i+\sum_{ij\in E^*}\beta_{ij}x_i x_j\Big\}\qquad x\in \{0,1\}^d,
$$
where $h_i,\beta_{ij}\in \R$ and $Z(h,\beta)$ is the normalizing constant. The corresponding model is called the Ising model on $T^*$.

%Denote the set of all distributions on $\cX$  that are Markov to $T^*$ by $M_\cX(T^*)$. If $r_i=r$ for all $i\in V$, we also write $M_r(T^*)$.

The set of distributions that are Markov with respect to $T^*$ can be equivalently described by the following Markov process on the tree $T^*$. Fix any \emph{inner} node $\rho\in V$,  call it the root, and direct all edges of $T^*$ away from $\rho$. Denote by $p_\rho$ the marginal distribution of $X_\rho$ and, for each edge $u\to v$, let $M^{uv}$  be the matrix representing the conditional distribution $p_{v|u}$ of $X_v$ given $X_u$; $M_{x_u,x_v}^{uv}=p_{v|u}(x_v|x_u)$ for $x_u\in \cX_u$, $x_v\in \cX_v$. Then
$$
p(x)\;=\;p_\rho(x_\rho)\prod_{u\to v}p_{v|u}(x_{v}|x_{u})\qquad\mbox{for all }x\in \cX.
$$
Thus, fixing a directed version of $T^*$ fixes a parameterization of the set of all distributions Markov to $T^*$ making it into a parametric statistical model.
\begin{rem}
For the above argument $\rho$ did not have to be an inner node of $T^*$. The fact that $\rho$ is assumed to be an inner node will simplify our theory in later sections.
\end{rem}

Suppose that $X\in \cX$ has distribution $p$ that is Markov  with  respect to $T^*$. Given a random sample from $p$, the goal is to recover the underlying tree. As we mentioned in the Introduction, this problem can be solved very efficiently both from the computational and statistical point of view by the Chow-Liu algorithm \cite{chow1968approximating}, which outputs the tree that maximizes the likelihood function. Maximizing other functionals like AIC or BIC is also possible \cite{edwards2010selecting}. As we see next, the problem of structure recovery becomes more complicated in presence of corrupted data, which is the focus of this paper.

\subsection{Noisy tree  distributions}

Assume now that the vector $X$ is not observed directly. Instead, we observe $X^e=(X_1^e,\ldots,X_d^e)$, a corrupted version of $X$. Here the only crucial assumption is that for every  $i\in V$ the distribution of $X_i^e$ depends on $X$ only  through the value of $X_i$.

The  simplest  corruption model is a direct generalization of the one used for the Ising models in \cite{katiyar2020robust}: $X_i$ gets corrupted with some probability $q_i$ and, if that happens, the corrupted value takes uniformly any of the remaining values. In other words, for every $i=1,\ldots,d$ and $k,l\in \cX_i$
$$
\P(X_i^e=l|X_i=k)\;=\;\begin{cases}
	1-q_i & \mbox{if }l=k,\\
	\tfrac{q_i}{r_i-1} & \mbox{if }l\neq k.
\end{cases}
$$
It is sensible to assume that $q_i$ is relatively small but our main results do not rely on this assumption. In  fact, we consider a much  more general corruption model given, for each $i\in V$, by any square stochastic matrix $M^i=[p_{kl}^i]$ with
\begin{equation}\label{eq:gencorr}
\P(X_i^e=l|X_i=k)\;=\;p_{kl}^i.	
\end{equation}
%\begin{rem}
%Nothing stops us from defining continuous corruption models. A particularly simple analysis is possible in the binary case when $X_i\in \{0,1\}$ and $X_i^e\in [0,1]$. We study this situation in Section~\ref{sec:conntinuous}.  	
%\end{rem}

Our problem can be therefore formulated as follows. Given the distribution of the corrupted version $X^e$ of $X$ recover (i) the underlying tree $T^*$, and (ii) the underlying distribution of $X$. As we argue in the next section, this problem can be naturally formulated in the language of latent tree models. The resulting links with phylogenetics provide new insights and a rich resource of relevant results that establish conditions under which $T^*$ can be recovered from the noisy data.

\section{Link to phylogenetics}\label{sec:phylo}

\subsection{Latent tree models}

Given a tree $T=(W,E)$ with nodes $W$ and edges $E$, the underlying tree model for the random vector $Y$ with values in the discrete space $\cY=\prod_{i\in W}\cY_i$ is the set of all distributions over $\cY$ that are Markov with respect to $T$ as defined in Section~\ref{sec:treeMarkov}. Suppose now that $$L=\{1,\ldots,d\}\subset W$$ is the set of vertices of $T$ corresponding to the leaves of $T$ (vertices of degree one). The set of marginal distributions of $X:=Y_L$ is called the \emph{latent tree model} over $T$ and denoted $M_\cY(T)$. For a more detailed discussion see Section 1.1 in \cite{zwiernik2018latent}.

%\subsection{General Markov models}

In general, the theory of latent tree models can be quite complicated; see \cite{zhang2004hierarchical}. In this paper we restrict to the most tractable case where the cardinality of each $\cY_i$ is the same, $|\cY_i|=r$ for every $i\in W$. In this case the corresponding latent tree model is often called the \textit{general Markov model} and we denote it by $M_r(T)$.

Our problem of recovering $T^*$ and the underlying distribution from the noisy observations $X^e$ is very closely connected to the classical problem of recovering $T$ in a latent tree model $M_r(T)$. Before we explain this connection in Section~\ref{sec:noisyaslatent}, we first recall the corresponding classical results following \cite{chang1996full}.
\begin{defn}
A class of matrices $\cM$ is \emph{reconstructible from rows} if for each $M\in \cM$ and each permutation matrix $P\neq I$, we have $PM \notin \cM$.	
\end{defn}
A natural subset of  square  matrices that is reconstructible from rows is obtained by restricting the diagonal entries to dominate the other entries in the corresponding column. We also  formulate the following assumptions on a latent tree model $M_r(T)$:
\begin{enumerate}
	\item [(A0)] $T$ has no nodes of degree two.
	\item [(A1)] The root distribution satisfies $p_\rho(x_\rho)>0$ for all $x_\rho\in \{0,\ldots,r-1\}$.
	\item [(A2)] For each edge $u\to v$ the transition matrix $M^{uv}=[p_{v|u}(x_v|x_u)]$ is invertible and it is not a permutation matrix.
\item [(A3)] For each edge $u\to v$ the transition matrix $M^{uv}$ is  reconstructible from rows.
\end{enumerate}

The following result follows immediately from Proposition~3.1 and Theorem~4.1 in \cite{chang1996full}.
\begin{thm}\label{th:chang}
Under the assumptions (A0)-(A2) on $p\in M_r(T)$ the underlying tree $T$ is uniquely identified from $p$. If, in addition, (A3) holds then the underlying parameters are identifiable too.
\end{thm}

\begin{rem}
Our formulation of this	result slightly differs from the original of Joseph  T. Chang \cite{chang1996full}. In his version of (A1) he assumes that  the positivity  condition holds for some node and we require this condition specifically for the root. Together with (A2) both versions are equivalent.\end{rem}

%Recovering the topology can be done in a simple way by mapping these distributions to the space of tree metrics. We provide this map in Section~\ref{sec:distance}  but first  we discuss the general identifiability results available in the phylogenetic literature.
%

\subsection{Trees with degree-two nodes}

Identifiability results developed in phylogenetics, like the one above, play a crucial role in this paper. In our situation however it is important to consider the case where the condition (A0) does not hold. Then, the tree can never be recovered uniquely. For a simple illustration consider two models, one of a single edge $\bullet\!-\!\bullet$ and one on the chain $\bullet\!-\!\circ\!-\!\bullet$, where the middle vertex represents a latent variable. In case all three variables involved have $r$ states, the family of distributions over the solid nodes in both models is the same.
\begin{defn}\label{def:contract}
If $u,v$ are two nodes in $T$ of degree different than two and such that each node on the unique path between them has degree two then by \emph{ suppressing these degree two nodes} we mean removing all these intermediate nodes together with all adjacent edges and adding a direct edge between $u$ and $v$.
\end{defn}

Denote by $\overline T$ the tree obtained from $T$ by suppressing all the degree two nodes. The following result can be found, for example, in Section~5.3.4 in \cite{LTbook}.
\begin{prop}\label{prop:suppress}
For any tree $T$, $M_r(T)\;=\;M_r(\overline T)$. If $p\in M_r(T)$ satisfies (A1)-(A3) then the same distribution in $M_r(\overline T)$ satisfies (A0)-(A3).
\end{prop}

\begin{rem}\label{rem:contract}
	Although the models $M_r(T)$ and $M_r(\overline T)$ are equal by Proposition~\ref{prop:suppress}, their parametrizations are not, as generally $T$ has more vertices and edges than $\overline T$. However, if $T$ is rooted at any node of degree different than two, the parameters of $M_r(\overline T)$ can be easily recovered from the parameters of $M_r(T)$. In both cases the root distribution $p_\rho$ is the same. For each edge $u\to v$ in $\overline{T}$ we also have the same transition matrix $M^{uv}$ unless $u \rightarrow v $  is an edge in $\overline{T}$ that has been obtained by suppressing degree two nodes $w_1,\ldots,w_k$ in a path $u \rightarrow w_1\rightarrow\cdots\rightarrow w_k  \rightarrow v$; in this case the transition matrix $M^{uv}$ is the product of transition matrices in that path, $M^{uw_1}\cdots M^{w_k v}$.
\end{rem}

\subsection{Distance based methods}\label{sec:distance}

The tree structure recovery in Theorem~\ref{th:chang} can be in fact done using only pairwise marginal distributions and this fact has important consequences for the rest of the paper. For any edge $u\to v$ denote by $P^{uv}$ the $r\times r$ matrix of the marginal distribution of $(X_u,X_v)$, and by $P^{uu}$ a diagonal matrix with the marginal distribution of $X_u$ on the diagonal. For any two vertices $u,v$ let
\begin{equation}\label{eq:u}
\tau_{uv}\;:=\;\frac{\det(P^{uv})}{\sqrt{\det(P^{uu}P^{vv})}},\end{equation}
where the denominator is non-zero if all marginal distributions are strictly positive. By essentially the same argument as in~\cite[Theorem 8.4.3]{semple2003phylogenetics} we obtain the following path-product formula
\begin{equation}\label{eq:qij}
\tau_{ij}\;\;\;=\;\;\prod_{(u,v)\in \overline{ij}} \tau_{uv}\qquad\mbox{for all }i,j\in W,
\end{equation}
where $\overline{ij}$ denotes the unique path between $i$ and $j$ in $T$.
%where
%$$
%\tau_{uv}=\det M^{uv}\sqrt{\frac{\det (P^{uu})  }{\det(P^{vv})}}.
%$$
\begin{rem}
In the case of binary variables, $\det P^{ij}={\rm cov}(X_i,X_j)$, $\det(P^{ii})={\rm var}(X_i)$ and so $\tau_{ij}$ is the correlation ${\rm corr}(X_i,X_j)$. 	
\end{rem}
It can be shown (c.f. Section~2.2 in \cite{zwiernik2018latent}) that
$$
	\tau_{uv}^2\;=\;\det M^{uv} \det M^{vu}.
	$$
	Because both $M^{uv}$ and $M^{vu}$ are stochastic matrices, all their eigenvalues lie in the unit circle. In particular, $\tau_{uv}\in [-1,1]$ and it is equal to $\pm 1$ precisely when $M^{uv}$ is a permutation matrix, or in other words, if $X_u$ and $X_v$ are functionally related. With assumptions (A1) and (A2) we have thus that $\tau_{uv}^2\in (0,1)$. Define $$d_{uv}:=-\log(\tau_{uv}^2)>0$$ then (\ref{eq:qij}) implies that
\begin{equation}\label{eq:dij}
d_{ij}\;\;\;=\;\;\sum_{uv\in \overline{ij}} d_{uv}\qquad\mbox{for all }i,j\in W.
\end{equation}
In other words $d_{uv}$ represent lengths of edges in the tree $T$ and $d_{ij}$ are then distances between vertices calculated by summing the lengths of edges on the unique path between them in $T$. The collection of distances between the leaves $D=[d_{ij}]_{i,j\in L}$ is called a \emph{tree metric}.

The following classical result assures that  $T$ can be recovered from the underlying tree metric; see Theorem~1 in \cite{buneman1971recovery}.
\begin{thm}[Buneman]\label{th:buneman}If $T$ (with leaves $L$) has no degree two nodes and $d_{uv}>0$ for every edge $uv$ of $T$. Then $T$ can be uniquely recovered from the tree metric $D=[d_{ij}]_{i,j\in L}$.	
\end{thm}
As we mentioned above the assumptions of this theorem are automatically satisfied for $d_{ij}=-\log\tau_{ij}$ in the general Markov model as long as the assumptions (A0), (A1), and (A2) hold. We finish this section giving the explicit link between latent tree models and tree models for corrupted data.

\subsection{The noisy tree model as a latent tree model}\label{sec:noisyaslatent}

 As in Section~\ref{sec:problem}, consider a  tree $T^*$  representing a random vector $X=(X_1,\dots,X_d)$ and assume that each $X_i$ can take $r$ states (from now on $r_i=r$ for $i=1,\dots, d$). We assume that the distribution of $X$ is Markov with respect to $T^*$, so in particular we can fix a root at an inner node $\rho$ and consider transition matrices at the directed edges.

We let $T^e$ be the tree obtained from $T^*$ by adding $d$ extra vertices representing the noisy variables $X_i^e$ and by linking each $X_i$ with $X_i^e$ by an edge with the corresponding transition matrix $M^i$; see Figure~\ref{fig:tree} for an example. Then $T^e$ is a rooted tree with the root $\rho$. Note that $X_i^e$ is independent of $\{X_j,X_j^e:\;j\neq i\}$ given $X_i$ and so the vector $(X,X^e)$ is Markov with respect to the augmented tree $T^e$. Consequently, the distribution of $X^e$ lies in the latent tree model $M(T^e)$.
\begin{prop}\label{prop:ingenmark}
	If $X$ has a distribution that is Markov to $T^*$, then $X^e$ has a distribution $p$ that lies in the general Markov model $M_r(T^e)$.
\end{prop}
To recover $T^*$ from the distribution of $X^e$ we first try to recover $T^e$. For that, note that $T^e$ has a special topology with each inner vertex having one and only one leaf-child. The degree two nodes in $T^e$ correspond precisely to the leaves of $T^*$. This special topology of $T^e$ plays a crucial role in the rest of this paper.

\begin{figure}
\tikzstyle{vertex}=[circle,fill=black,minimum size=5pt,inner sep=0pt]
\tikzstyle{hidden}=[circle,draw,minimum size=5pt,inner sep=0pt]
\begin{tikzpicture}[scale=.7]
  \node[hidden] (1) at (0,0)  [label=above:${\rho=1}$] {};
    \node[hidden] (2) at (1,-1.5) [label=right:$2$]{};
    \node[hidden] (3) at (-1.5,-3) [label=left:$3$]{};
    \node[hidden] (4) at (0.5,-3) [label=left:$4$]{};
    \node[hidden] (5) at (1.5,-3) [label=right:$5$]{};
    \draw[line width=.3mm] (1) to (2);
    \draw[line width=.3mm] (1) to (3);
    \draw[line width=.3mm] (2) to (4);
        \draw[line width=.3mm] (2) to (5);
  \end{tikzpicture} 	\qquad \tikzstyle{vertex}=[circle,fill=black,minimum size=5pt,inner sep=0pt]
\tikzstyle{hidden}=[circle,draw,minimum size=5pt,inner sep=0pt]
\begin{tikzpicture}[scale=.7]
  \node[hidden] (1) at (0,0)  [label=above:${\rho=1}$] {};
    \node[hidden] (2) at (1,-1.5) [label=right:$2$]{};
    \node[hidden] (3) at (-1.5,-3) [label=left:$3$]{};
    \node[hidden] (4) at (0.5,-3) [label=left:$4$]{};
    \node[hidden] (5) at (1.5,-3) [label=right:$5$]{};
      \node[vertex] (1a) at (-1,-4.5)  [label=below:$1^e$]{};
    \node[vertex] (2a) at (1,-4.5) [label=below:$2^e$]{};
    \node[vertex] (3a) at (-2,-4.5) [label=below:$3^e$]{};
    \node[vertex] (4a) at (0,-4.5) [label=below:$4^e$]{};
    \node[vertex] (5a) at (2,-4.5) [label=below:$5^e$]{};
    \draw[line width=.3mm] (1) to (2);
    \draw[line width=.3mm] (1) to (3);
    \draw[line width=.3mm] (2) to (4);
    \draw[line width=.3mm] (2) to (5);
    \draw[line width=.3mm] (1a) to (1);
    \draw[line width=.3mm] (2a) to (2);
    \draw[line width=.3mm] (3a) to (3);
    \draw[line width=.3mm] (4a) to (4);
    \draw[line width=.3mm] (5a) to (5);
  \end{tikzpicture} 	
  \caption{A tree $T^*$ on the left and the augmented tree $T^e$ on the right. Solid nodes represent corrupted observations.}\label{fig:tree}
\end{figure}
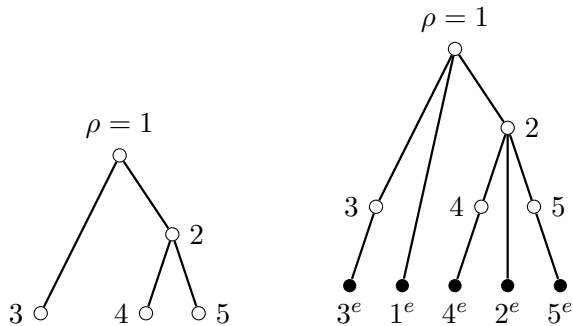

%In our problem formulation we observe only $X^e$ and the vector $X$ remains hidden. In other words, the observed data give us information about the marginal distribution of $X^e$. Since $(X,X^e)$ is Markov with  respect to $T^e$ and $X^e$ is represented by the leaves of $T^e$, the marginal distribution of $X^e$ lies the latent tree model  $M(T^e)$.
%
%
%As explained earlier, to get a parameterization of this model, we fix a root $\rho$ at any inner node of $T^*$. The parameters of the model $M_\cX(T^e)$ are: the root distribution $p_\rho$, the  transition matrices $M^{uv}$ for all edges $u\to v$ in $T^*$, and $d$ additional transition matrices $M^i$ for $i\in V$ that define the  corruption model in (\ref{eq:gencorr}). Our aim is to recover $T^*$, the root distribution, and the transition matrices related to edges of $T^*$. In its full generality this problem may  get complicated. For this reason, we focus on the special case where $r_i=r$ for all $i=1,\ldots,r$, which directly generalizes the Ising model. In this case the conditional distributions $p_{uv}(x_v|x_u)$ are represented by square stochastic matrices $M^{uv}$. We now describe these models in more detail.
%
%

\section{Identifying $T^*$ from corrupted data}\label{sec:identify}

Now that we linked noisy tree models to latent tree models, identifiability results follow from the theory developed in mathematical phylogenetics. We exploit in addition the special form of the topology of $T^e$.

\subsection{The equivalence class of $T^*$}

In our case, the tree $\overline{T^e}$  is obtained by suppressing in  $T^e$ the nodes that correspond to the leaves of $T^*$; c.f. Definition~\ref{def:contract}. For the tree in Figure~\ref{fig:tree}, the tree $\overline{T^e}$ is given on  the left in Figure~\ref{fig:tree2}.	

Recall that in this new language, the goal is to recover $T^*$ from a distribution $p\in M_r(T^e)$.
\begin{thm}\label{th:overlineT}
If $p\in M_r(T^e)$ satisfies (A1)-(A2) then the tree $\overline{T^e}$ is uniquely identified from $p$. If, in addition, $p$ satisfies (A3) then the  underlying parameters of the model $M_r(\overline{T^e})$ are uniquely identified too.
\end{thm}
\begin{proof}
By Proposition~\ref{prop:ingenmark}, $X^e$ has distribution in $M_r(T^e)$. By Proposition~\ref{prop:suppress}, $M_r(T^e)=M_r(\overline{T^e})$. Moreover, if $p\in M_r(T^e)$ satisfies (A1)-(A2) then $p\in M_r(\overline{T^e})$ satisfies (A0)-(A2). By Theorem~\ref{th:chang}, the underlying tree $\overline{T^e}$ can be uniquely identified. The same conclusion holds for identifying the parameters of $M_r(\overline{T^e})$ if (A3) holds too.
\end{proof}

Denote by $[T^*]$ the set of all trees $S$ over the vertex set $V=\{1,\ldots,d\}$ such that $\overline{S^e}=\overline{T^e}$. {Here we mean equality as semi-labelled trees, that is, $\overline{S^e}$ and $\overline{T^e}$ must have the same topology and labelling of the leaf nodes but the labelling of the inner nodes is irrelevant.} Directly  by construction, $T^*\in [T^*]$. For another example, let $T^*$ be the tree on the  left in Figure~\ref{fig:tree}, where the corresponding tree $T^e$ is given on the right. The tree $\overline{T^e}$ is given on the left in  Figure~\ref{fig:tree2}. Now let $S$ be a tree like $T^*$ but 3 swapped with 1 and 2 swapped with 4. The corresponding tree $S^e$ is depicted on the right in Figure~\ref{fig:tree2} and $\overline{S^e}=\overline{T^e}$.

By Proposition~\ref{prop:suppress}, if $S\in [T^*]$ then $$M_r(T^e)\;=\;M_r(\overline{T^e})\;=\;M_r(S^e)$$
and so we cannot distinguish from the corrupted data  between the trees in $[T^*]$ because each $S\in [T^*]$ leads to the same model $M_r(\overline{T^e})$ for $X^e$. Theorem~\ref{th:overlineT} implies the following result.
\begin{thm}\label{th:main}
If $p\in M_r(T^e)$ satisfies (A1)-(A2), then $T^*$ can be recovered from $p$ up to the equivalence class $[T^*]$.	
\end{thm}

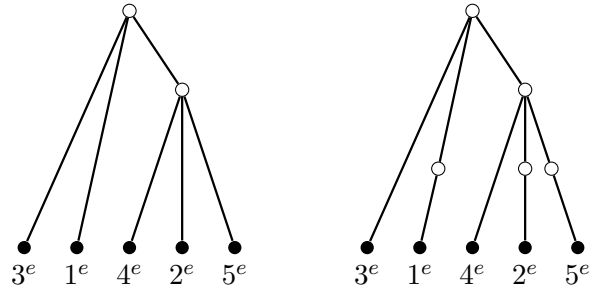
\begin{figure}
\tikzstyle{vertex}=[circle,fill=black,minimum size=5pt,inner sep=0pt]
\tikzstyle{hidden}=[circle,draw,minimum size=5pt,inner sep=0pt]
\begin{tikzpicture}[scale=.7]
  \node[hidden] (1) at (0,0)   {};
    \node[hidden] (2) at (1,-1.5) {};
      \node[vertex] (1a) at (-1,-4.5)  [label=below:$1^e$]{};
    \node[vertex] (2a) at (1,-4.5) [label=below:$2^e$]{};
    \node[vertex] (3a) at (-2,-4.5) [label=below:$3^e$]{};
    \node[vertex] (4a) at (0,-4.5) [label=below:$4^e$]{};
    \node[vertex] (5a) at (2,-4.5) [label=below:$5^e$]{};
    \draw[line width=.3mm] (1) to (2);
    \draw[line width=.3mm] (1) to (3a);
    \draw[line width=.3mm] (2) to (4a);
    \draw[line width=.3mm] (2) to (5a);
    \draw[line width=.3mm] (1a) to (1);
    \draw[line width=.3mm] (2a) to (2);
  \end{tikzpicture} 	
  \qquad \tikzstyle{vertex}=[circle,fill=black,minimum size=5pt,inner sep=0pt]
\tikzstyle{hidden}=[circle,draw,minimum size=5pt,inner sep=0pt]
\begin{tikzpicture}[scale=.7]
  \node[hidden] (1) at (0,0)   {};
    \node[hidden] (2) at (1,-1.5) {};
    \node[hidden] (3) at (-0.65,-3) {};
    \node[hidden] (4) at (1,-3) {};
    \node[hidden] (5) at (1.5,-3) {};
      \node[vertex] (1a) at (-1,-4.5)  [label=below:$1^e$]{};
    \node[vertex] (2a) at (1,-4.5) [label=below:$2^e$]{};
    \node[vertex] (3a) at (-2,-4.5) [label=below:$3^e$]{};
    \node[vertex] (4a) at (0,-4.5) [label=below:$4^e$]{};
    \node[vertex] (5a) at (2,-4.5) [label=below:$5^e$]{};
    \draw[line width=.3mm] (1) to (2);
    \draw[line width=.3mm] (1) to (3a);
    \draw[line width=.3mm] (1) to (3);
    \draw[line width=.3mm] (2) to (4);
    \draw[line width=.3mm] (2) to (5);
    \draw[line width=.3mm] (1a) to (3);
    \draw[line width=.3mm] (2a) to (4);
    \draw[line width=.3mm] (4a) to (2);
    \draw[line width=.3mm] (5a) to (5);
  \end{tikzpicture} 	

  \caption{The tree $\overline T^e$ for $T^* $ in Figure~\ref{fig:tree} and one of the trees $S^e$ for in $S\in[T^*]$.}\label{fig:tree2}
\end{figure}

The tree $\overline{T^e}$ is one natural way of representing the equivalence class $[T^*]$. To have a concrete description of this equivalence class directly in terms of $T^*$ call an inner node of $T^*$  a \emph{mother} if it is adjacent to at least one leaf of $T^*$. For example, the root and the node 2 in the tree $T^*$ in Figure~\ref{fig:tree} are mothers.
\begin{prop}\label{prop_uptolabel}
Let $A$ be the set of mothers in ${T}^*$. Then $S\in [T^*]$ if and only if $S$ is obtained from $T^*$ by label swapping of each node in $A$ and its adjacent leaves. In particular, the equivalence class $[T^*]$ is the same as the one defined in \cite{katiyar2020robust}.
\end{prop}

\begin{proof}
All inner nodes of $T^e$ have exactly one adjacent leaf. Passing to $\overline{T^e}$ this changes only for the mother nodes, whose leaves represent the noisy version of the mother node and the noisy versions of all its leaves (in $T^*$). {The same occurs for any tree $S$ in $[T^*]$, as $\overline{S^e}=\overline{T^e}$. By shrinking exactly one leaf edge for each inner node of $\overline{T^e}$ (by \emph{shrinking} we mean removing the leaf edge and putting the leaf label as the label of the corresponding inner node), we recover all inner nodes of $T^*$ (resp. $S$) except for the mother nodes.}
 %To recover any other $S$ such that $\overline{S^e}=\overline{T^e}$ we shrink exactly one leaf edge for each inner node of $\overline{T^e}$ (by \emph{shrinking} we mean removing the leaf edge and putting the leaf label as the label of the corresponding inner node).
 Thus, we can identify the node labels in $T^*$ up to label swapping of each mother in $A$ and its adjacent leaves. The last part of the statement follows from the fact that the set of mother nodes is precisely the set $\cA$ defined in Section~3 in \cite{katiyar2020robust}.
\end{proof}

As a byproduct of Theorem \ref{th:main} and this characterization of the class of $[T^*]$ we obtain:
\begin{cor}
The unlabelled version of $T^*$ can be always correctly identified {from a distribution $p\in M_r(T^e)$ satisfying $(A1)-(A2)$}.	
\end{cor}
%Using Proposition ~\ref{prop:distuniqu}, we get the same result when we learn the tree only from the  underlying tree metric.
%\begin{prop}\label{prop:identdist}
%	Let $A$ be the set of mothers in ${T}^*$. Then one can identify the tree $T^*$ from the tree metric induced by the distribution of $X^e$ (c.f. Section~\ref{sec:distance}) up to label swapping of each node in $A$ and its adjacent leaves.
%\end{prop}
%

We next discuss two somewhat extreme examples that show that $[T^*]$ can be large or small depending on $T^*$.
\begin{ex}
	If $T^*$ forms a chain $1-2-\cdots-d$ with $d\geq 4$  then $T^*$ has two mothers: $2$ and $d-1$. By Proposition~\ref{prop_uptolabel}, the equivalence class $[T^*]$ contains four trees where the  pairs of labels $1,2$ and $d-1,d$ are potentially swapped.
\end{ex}

\begin{ex}
	If $T^*$ is a star tree with 1 in the center and $d-1$ leaves $2,\ldots,d$ then $T^*$ has a single mother $1$. Since each vertex is a adjacent to $1$, it follows by Proposition~\ref{prop_uptolabel} that $[T^*]$ contains $d$ star trees with any of the $d$ vertices of $T^*$ being a potential center.
\end{ex}

%	\begin{prop}
%		If  $T^*$ has $m$ mothers and $k$ leaves then the number of elements in $[T^*]$ equals $m+k$.
%	\end{prop}
%	
%	\begin{proof}
%		Let $A$ be the set of mothers of $T^*$. For each $u\in A$, suppose $k_u$ is the number of leaves adjacent to $u$. By definition, each tree $S$ obtained by swapping $u$ with one of its adjacent leaves belongs  to $[T^*]$. The number of such trees equals $k_u + 1$. This can be done for each mother $u\in A$, and since each leaf is adjacent to a unique mother, the number of elements in $[T^*]$ is:
%		\begin{equation*}
%		\sum_{u\in A} \left(k_u + 1\right) = m + \sum_{u\in A} k_u = m+k.
%		\end{equation*}
%		\piotr{This should be $\prod_{u\in A} \left(k_u + 1\right)$, no? BTW \cite{katiyar2020robust} give this essentially in Section 3.}\marta{I agree. I write this proposition again below as a Cor. If we agree, we should remove  the Prop.}
%	\end{proof}
\begin{cor}
		If $A$ is the set of mothers of $T^*$ and each mother $u$ in $A$ has $k_u$ adjacent leaves, then the number of trees in $[T^*]$ equals $\prod_{u\in A} \left(k_u + 1\right)$.
	\end{cor}
	\begin{proof} By Proposition \ref{prop_uptolabel}, the trees in $[T^*]$ are obtained by label swapping of the nodes $u\in A$ with their adjacent leaves. For each node $u\in A$, we have $k_u+1$ labels to swap (counting the label of $u$ and its adjacent leaves). Thus, there are $\prod_{u\in A} \left(k_u + 1\right)$ trees in $[T^*]$.
\end{proof}

This shows that the equivalence class $[T^*]$ can be still potentially quite large. We now discuss extra assumptions that allow us to identify $T^*$ uniquely.

\subsection{Identifying $T^*$ exactly}\label{sec:exact}

It is possible to completely identify the tree $T^*$ under a mild assumption on the noise of the mother nodes. We formulate this condition in terms of the distances $d_{ij}=-\log\tau_{ij}^2$ defined in Section~\ref{sec:distance}.

\begin{thm}\label{th:identifyexactly} Let $X$ be a distribution that is Markov with respect to $T^*$ and let $X^e$ be a corrupted version of $X$ with a distribution $p\in M_r(T^e)$.
%with parametersbution $p_{\rho}$ at the root and transition matrices $M^{u v}$ for each edge $u \rightarrow v$ in $T$.
Assume that $p$ satisfies (A1), (A2) and

(A4) for each mother $u$ in $T^*$, $d_{u,i^e}>d_{u,u^e}$ for each leaf $i$ adjacent to $u$.

\noindent Then, the tree $T^*$ is uniquely identifiable.
{If, in addition, $p$ satisfies $(A3)$, then the underlying parameters corresponding to internal edges of $T^*$ can also be uniquely identified.}
\end{thm}
\begin{proof}
By Theorem~\ref{th:overlineT} we can recover the underlying tree $\overline{T^e}$ from the distances implied by $p\in M_r(\overline{T^e})$. Then it is also straightforward to identify the underlying edge lengths $d_{uv}$ for the edges $uv$ of $\overline{T^e}$. To recover $T^*$ we need to shrink exactly one terminal edge for each inner node of $\overline{T^e}$; c.f. the proof of Proposition~\ref{prop_uptolabel}. The only ambiguity in recovering $T^*$ from $\overline{T^e}$ comes from the nodes that have more than one leaf (corresponding to the mother nodes in $T^*$). If $u$ is a given mother node then the lengths of the corresponding terminal edges are $d_{u,u^e}$ and $d_{u,i^e}$ for all leaves $i$ adjacent to $u$ (in $T^*$). With our constraints, the node $u^e$ is the one of minimum distance to $u$. Thus, the tree $T^*$ can be completely identified. {The last statement follows directly from Theorem \ref{th:main}.}
\end{proof}

In terms of the parameters of the distribution, condition (A4) translates to
\begin{eqnarray*}
\det M^{u \, u^e}\det M^{u^e \, u} &>& \det M^{i^e u}\det M^{u \, i^e}\\
&= &(\det M^{i u}\det M^{u \, i})(\det M^{i^e i}\det M^{i \, i^e}),	
\end{eqnarray*}
where $M^{uv}$ denotes the conditional distribution of the variable represented by the node $v$ given the variable represented by the node $u$. In particular, for the Ising model on $T^*$, this is equivalent to saying that
$$
|{\rm corr}(X_u,X_u^e)|\geq |{\rm corr}(X_u,X_i){\rm corr}(X_i,X_i^e)|
$$
for each mother $u$ and each of its adjacent leaves $i$. In other words, the correlation between $X_u$ and its noisy version is greater than the correlation between $X_u$ and the noisy version of any other variable; hardly a controversial assumption to make.

\section{Learning $T^*$ and its parameters}\label{sec:treelearning}

In this section we briefly review some of the methods that can be used to learn the tree  $\overline{T^e}$ that represents the equivalence class $[T^*]$ from data. We also show how this can be extended to learn the parameters of $M_r(\overline{T^e})$ and how it affects the problem of learning $T^*$ and the corresponding parameters. We demonstrate the performance of some of these methods on simulated data.

\subsection{Consistency of the Chow-Liu algorithm}

Recall that the Chow-Liu algorithm \cite{chow1968approximating} relies on computing mutual informations and building the maximum cost spanning tree of the resulting weighted graph. Following the debate in \cite{katiyar2020robust} on whether the Chow-Liu method is a good method to recover the tree structure for noisy data, we study the conditions under which noisy data still allow for consistent estimation of the correct tree in some specific cases. We focus on the situation when for every $i,j$ the mutual information $I(X_i, X_j)$ between $X_i$ and $X_j$ is a strictly decreasing function of the distance $d_{ij}=-\log\tau_{ij}^2$. This includes the binary Ising model with no external field and more generally the \textit{fully symmetric tree models on $r$ states}; see Lemma~6 in \cite{choi2011learning}. The fully symmetric model on $r$ states is a tree model such that each variable has $r$ states and uniform marginal distribution. Moreover, each transition matrix is of the form
$$
M^{uv}_{ij}\;=\;\begin{cases}
	1-(r-1)\theta_{ij} & \mbox{if }i=j,\\
	\theta_{ij} & \mbox{otherwise}.
\end{cases}
$$

In the special case when the mutual informations $I(X_i,X_j)$ are strictly decreasing function the distances $d_{ij}=-\log\tau^2_{ij}$, we can equivalently build the minimum cost spanning tree based on the distances $D=[d_{ij}]$. For consistency argument, we can replace the sample correlations $\hat \rho_{ij}$ with the actual correlations of the data generating distribution. In this case $D=[d_{ij}]$ forms a tree metric on $T^*$ and $T^*$ is the (unique) minimum cost spanning tree of the complete graph weighted with $D$; we write $T^*={\rm MWST}(D)$. This implies that the Chow-Liu algorithm is a consistent tree recovery method.

To consider consistency of the Chow-Liu algorithm for noisy data note that now the corresponding distances are $\bar d_{ij}=-\log \bar \tau_{ij}^2$ where $\bar \tau_{ij}$ are defined for $(X_i^e,X_j^e)$. We have $$\bar d_{ij}=d_{ij}+d_{i i^e}+d_{j j^e}$$
and so $\overline D=[\bar d_{ij}]$ does not form a tree metric on $T^*$. In the next theorem we provide conditions on the noise distribution that assure that the Chow-Liu method remains a consistent method for the unique recovery of the true underlying tree $T^*$. In other words, we study the conditions under which $T^*={\rm MWST}(\overline D)$. We will use the notation $\ell_i:=d_{i i^e}=-\log(\tau_{i i^e}^2)$.
\begin{prop}\label{th:CL}
Suppose that for every $i,j$ the mutual information $I(X_i,X_j)$ is a decreasing function of $d_{ij}=-\log\tau_{ij}^2$. Then for the Chow-Liu method to be a consistent tree recovery method with noisy data it must hold that  $d_{uv}\geq  \ell_u-\ell_v$ for all edges $uv$ of $T^*$ such that $u$ is not a leaf of $T^*$. On the other hand, if all these inequalities are strict, this condition is also sufficient for unique recovery of $T^*$.
\end{prop}
\begin{proof}
	We first show that the condition in the theorem is necessary. Let $uv$ in $T^*$ be an edge such that $u$ is not a leaf. In this case there exists a node $w$ such that $w\neq v$ and $uw$ is an edge of $T^*$. Consider the cycle $(w,u,v)$. By the cycle property of the minimum weight spanning tree the condition $T^*={\rm MWST}(\overline D)$ implies that 	$\bar d_{uv}\leq \bar d_{vw}$ and $\bar d_{uw}\leq \bar d_{vw}$, which translates to $d_{uw}\geq  \ell_u-\ell_w$ and $d_{uv}\geq  \ell_u-\ell_v$, proving necessity.

Suppose now that the condition of the theorem holds with strict inequalities, that is,  $d_{uv}>  \ell_u-\ell_v$ for all edges $uv$ of $T^*$ such that $u$ is not a leaf of $T^*$. Let $i,j$ be any two non-adjacent vertices of $T^*$ and let
%$\overline{ij}$ denote the unique path between them in $T^*$. Let $k$ be any other node on this path.
$k$ be any other node on the path between them. We have $\bar d_{ik}=d_{ik}+\ell_i+\ell_k$, $\bar d_{jk}=d_{jk}+\ell_j+\ell_k$, and
$$
\bar d_{ij}=d_{ik}+d_{kj}+\ell_i+\ell_j.
$$
Similarly, as above we show that $\bar d_{ij}> \max\{\bar d_{ik},\bar d_{jk}\}$ as long as $d_{jk}> \ell_k-\ell_j$ and $d_{ik}> \ell_k-\ell_j$. This would then imply that $ij$ cannot be an edge in ${\rm MWST}(\overline D)$.  We show that $d_{jk}> \ell_k-\ell_j$ and the proof of the second inequality is similar. If $jk$ is an edge of $T^*$ then $d_{jk}> \ell_k-\ell_j$ because $k$ is not a leaf. If $jk$ is not an edge then there exists a path $k-i_1-\cdots-i_m-j$. Since $k,i_1,\ldots,i_m$ are non-leaves we conclude
$$
d_{jk}\;=\;d_{k{i_1}}+\cdots+d_{i_m j}\;>\; (\ell_{k}-\ell_{i_1})+\cdots+(\ell_{i_m}-\ell_{j})\;=\;\ell_k-\ell_j
$$
		\end{proof}

\begin{cor}
For any symmetric discrete tree model, if $\ell_i=\ell\geq 0$ for all $1\leq i\leq d$ then Chow-Liu gives a consistent way of uniquely recovering $T^*$ from the noisy data.
\end{cor}

Note that the condition $d_{uv}>\ell_u-\ell_v$ is equivalent to $d_{u v^e}>d_{u u^e}$, which is a natural assumption in many applications. In case this condition does not hold, Theorem~\ref{th:identifyexactly} assures that $T^*$ can be still uniquely identified as long as the condition $d_{uv}>\ell_u-\ell_v$ holds for all cases when $u$ is a mother node and $v$ is one of its leaves. It is just that this identifiability cannot be in general obtained using the Chow-Liu algorithm.

There is a handful of algorithms that can be used to learn $\overline{T^e}$ (or equivalently $[T^*]$); see, for example, \cite{choi2011learning,mossel2013robust}. In case unique recovery conditions in Theorem~\ref{th:identifyexactly} hold, we can recover $T^*$ from $\overline{T^e}$ by shrinking for each inner node the shortest of its terminal edges.%\piotr{More?}\marta{I think this is enough.}

\subsection{Learning $T*$ from distances by Neighbor-Joining}\label{sec:simulDiscrete}

One of the most widely used methods to recover a phylogenetic tree from evolutionary pairwise distances is Neighbor-Joining (briefly NJ) \cite{saitou1987}. In order to test the performance of this method in the recovery of $\overline{T^e}$  or $T^*$ from noisy data, we have simulated corrupted data on each of the trees $T^*$ of Figure \ref{fig:trees_simul}. We have restricted ourselves to the fully symmetric model on $r$ states for $r=2$ and $r=4$ (also known as the Jukes-Cantor model, \cite{JC69}) both for the stochastic matrices $M^{uv}$ and $M^i$. For each edge $u\rightarrow v$ of $T^*$, we have set the off-diagonal entries of $M^{uv}$  equal to $0.20$ for $r=2$ and to $0.07$ for $r=4$ so that the distance $d_{u,v}$ equals 1 for $r=2$ and $2$ for $r=4$. In order to see how the performance of the method varies when the noise increases, we have set all distances $d_{i,i^e}$ equal to $\ell$ and let $\ell$ vary from $0.01$ up to $3$ for $r=2$ and up to $4$ for $r=4$ in intervals of  $0.1$. Note that in practice we do not expect $\ell$ to be larger than $1$ ($X_i$ should be correlated more with its noisy version than with other variables). Finally, for each set of parameters we considered $1000$ experiments, each with sample size $5000$ from the corresponding tree distribution.

%\marina{For each set of parameters we considered $1000$ distributions, each corresponding to $5000$ independent samples from the corresponding multinomial distribution.}\piotr{Finally, for each set of parameters we considered $5000$ experiments, each with sample size $1000$ from the corresponding tree distribution.}

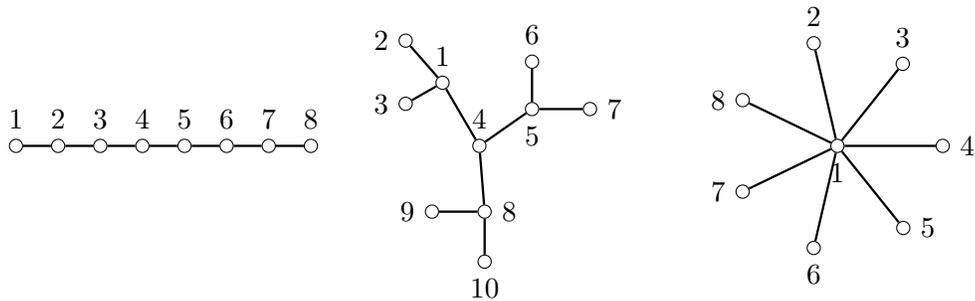
\begin{figure}
	\tikzstyle{vertex}=[circle,fill=black,minimum size=5pt,inner sep=0pt]
	\tikzstyle{hidden}=[circle,draw,minimum size=5pt,inner sep=0pt]
	\begin{tikzpicture}[scale=.7]
	\node[hidden] (9) at (-5,0)  [label=above:$8$] {};
	\node[hidden] (10) at (-5.8,0) [label=above:$7$]{};
	\node[hidden] (11) at (-6.6,0) [label=above:$6$]{};
	\node[hidden] (12) at (-7.4,0) [label=above:$5$]{};
	\node[hidden] (13) at (-8.2,0) [label=above:$4$]{};
	\node[hidden] (14) at (-9,0) [label=above:$3$]{};
	\node[hidden] (15) at (-9.8,0) [label=above:$2$]{};
	\node[hidden] (16) at (-10.6,0) [label=above:$1$]{};
	\draw[line width=.3mm] (9) to (10);
	\draw[line width=.3mm] (10) to (11);
	\draw[line width=.3mm] (11) to (12);
	\draw[line width=.3mm] (12) to (13);
	\draw[line width=.3mm] (13) to (14);
	\draw[line width=.3mm] (14) to (15);
	\draw[line width=.3mm] (15) to (16);

	\node[hidden] (1) at (-2.5, 1.2)  [label=above:$1$] {};
	\node[hidden] (2) at (-3.2, 2) [label=left:$2$]{};
	\node[hidden] (3) at (-3.2, 0.8) [label=left:$3$]{};
	\node[hidden] (4) at (-1.8, 0) [label=above:$4$]{};
	\node[hidden] (5) at (-0.8, 0.7) [label=below:$5$]{};
	\node[hidden] (6) at (-0.8, 1.6) [label=above:$6$]{};
	\node[hidden] (7) at (0.3, 0.7) [label=right:$7$]{};
	\node[hidden] (8) at (-1.7, -1.25) [label=right:$8$]{};
	\node[hidden] (9) at (-2.7, -1.25) [label=left:$9$]{};
	\node[hidden] (10) at (-1.7, -2.2) [label=below:$10$]{};
	\draw[line width=.3mm] (1) to (2);
	\draw[line width=.3mm] (1) to (3);
	\draw[line width=.3mm] (1) to (4);
	\draw[line width=.3mm] (4) to (5);
	\draw[line width=.3mm] (4) to (8);
	\draw[line width=.3mm] (5) to (6);
	\draw[line width=.3mm] (5) to (7);
	\draw[line width=.3mm] (8) to (9);
	\draw[line width=.3mm] (8) to (10);

	\node[hidden] (1) at (5,0)  [label=below:$1$] {};
	\node[hidden] (2) at (4.55, 1.95) [label=above:$2$]{};
	\node[hidden] (3) at (6.24, 1.56) [label=above:$3$]{};
	\node[hidden] (4) at (7, 0) [label=right:$4$]{};
	\node[hidden] (5) at (6.25, -1.56) [label=right:$5$]{};
	\node[hidden] (6) at (4.55, -1.94) [label=below:$6$]{};
	\node[hidden] (7) at (3.2, -0.87) [label=left:$7$]{};
	\node[hidden] (8) at (3.2, 0.87) [label=left:$8$]{};
	\draw[line width=.3mm] (1) to (2);
	\draw[line width=.3mm] (1) to (3);
	\draw[line width=.3mm] (1) to (4);
	\draw[line width=.3mm] (1) to (5);
	\draw[line width=.3mm] (1) to (6);
	\draw[line width=.3mm] (1) to (7);
	\draw[line width=.3mm] (1) to (8);

	\end{tikzpicture}
	
%	\tikzstyle{vertex}=[circle,fill=black,minimum size=5pt,inner sep=0pt]
%	\tikzstyle{hidden}=[circle,draw,minimum size=5pt,inner sep=0pt]
%	\begin{tikzpicture}[scale=2]
%	\node[hidden] (1) at (-0.4,0.4)  [label=above:$1$] {};
%	\node[hidden] (2) at (-0.67,0.16) [label=left:$2$]{};
%	\node[hidden] (3) at (-0.63,0.61) [label=above:$3$]{};
%	\node[hidden] (4) at (0,0) [label=above:$4$]{};
%	\node[hidden] (5) at (0.4,0.2) [label=below:$5$]{};
%	\node[hidden] (6) at (0.4,0.6) [label=above:$6$]{};
%	\node[hidden] (7) at (0.8,0.2) [label=right:$7$]{};
%	\node[hidden] (8) at (-0.2,-0.4) [label=right:$8$]{};
%	\node[hidden] (9) at (-0.2,-0.8) [label=right:$9$]{};
%	\node[hidden] (10) at (-0.6,-0.4) [label=left:$10$]{};
%	\draw[line width=.3mm] (1) to (2);
%	\draw[line width=.3mm] (1) to (3);
%	\draw[line width=.3mm] (1) to (4);
%	\draw[line width=.3mm] (4) to (5);
%	\draw[line width=.3mm] (4) to (8);
%	\draw[line width=.3mm] (5) to (6);
%	\draw[line width=.3mm] (5) to (7);
%	\draw[line width=.3mm] (8) to (9);
%	\draw[line width=.3mm] (8) to (10);
%	\end{tikzpicture}
	\caption{A chain tree on the left, a binary tree with ten nodes on the middle and a star tree on the right.}\label{fig:trees_simul}
\end{figure}

Using the normalized Robinson-Foulds distance (as implemented in \cite{phangorn}), we measure how far is the recovered tree from the original tree class $[T^*]$ (that is, from $\overline{T^e}$). Note that NJ always outputs binary trees while the trees $\overline{T^e}$ we are considering are not binary trees. Thus, we set a tolerance $\varepsilon$ such that if the estimated length of an internal edge is smaller than $\varepsilon$, we shrink that edge. In Figure \ref{fig:results_r24} we show the results for different values of $\varepsilon$. We observe that the best results are obtained when $\varepsilon$ is about half of the length of the branches of $T^*$. As expected, as the noise $\ell$ (the length of the corrupted branches) increases, it becomes more difficult to recover $\overline{T^e}$. {For example,when $\ell$ is smaller than the length of the original edges of the tree (i.e. $\ell\leq 1$ for $r=2$ or $\ell\leq 2$ for $r=4$), then we obtain highly successful results for almost all values of $\varepsilon$.}
%\marina{
It is worth noting that for $\ell=2$ we have transition matrices with condition number $2.72$ for $r=2$ and $1.4$ for $r=4$; similarly, for $\ell=3$ transition matrices have condition number $4.48$ for $r=2$ and $1.65$ for $r=4$. This has to be taken into account in relation to the hypothesis $(A2)$ about the invertibility of transition matrices (see Theorem \ref{th:overlineT}) and gives an insight to the different performance we obtain for $r=2$ and $r=4$.

%This fact affects differently to the performance of the reconstruction and to the Normalized Robinson-Foulds distance when $r=2$ or $r=4$.}
%Note that $l=4$ corresponds to transition matrices with determinant equal to $e^{-4}\sim 0.018$ \marina{(I think is $e^{-2}$)}, which has to be taken into account in relation to the hypothesis $(A2)$ about the invertibility of transition matrices (see Theorem \ref{th:overlineT}).\marta{it could be better to express this in terms of condition number, which will depend on the number of states. Marina, could you compute the condition numbers for these matrices? Also, could $\varepsilon$ get to $1.0$ for $r=4$?}

%\marina{
%\begin{table}[H]
%	\def\arraystretch{1.15}
%	\begin{tabular}{|c|c|c|c|}
%		\hline
%		\textbf{states}    & \textbf{l} & \textbf{det}       & \textbf{condition number} \\ \hline
%		\multirow{2}{*}{2} & 2          & $e^{-1}\sim0.37$                  & 2.72                      \\
%		& 3          & $e^{-3/2}\sim0.22$                  & 4.48                      \\ \hline
%		\multirow{3}{*}{4} & 2          & $e^{-1}\sim0.37$   & 1.4                       \\
%		& 3          & $e^{-3/2}\sim0.22$ & 1.65                      \\
%		& 4          & $e^{-2}\sim0.14$     & 1.95                      \\ \hline
%	\end{tabular}
%\end{table}
%}

\begin{figure}
	\centering
	\makebox[\textwidth][c]{\includegraphics[width=1.2\textwidth]{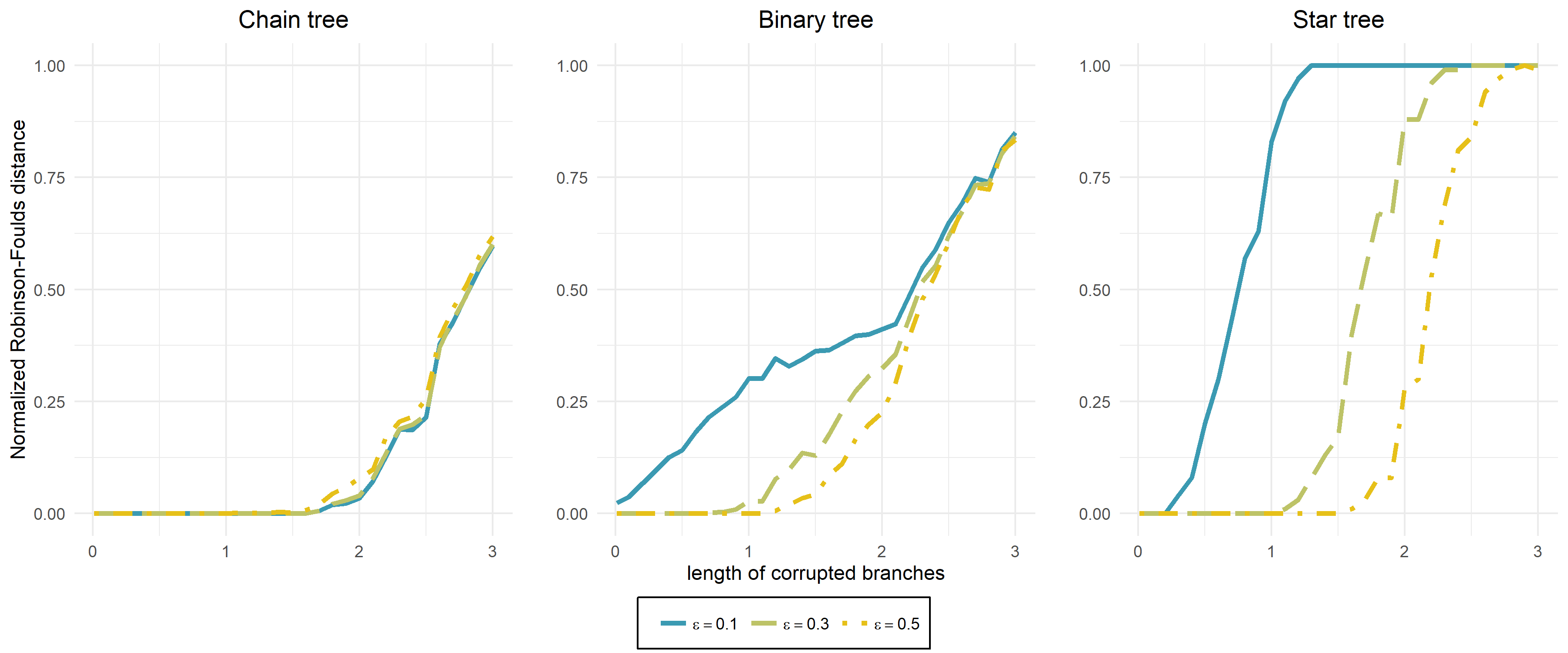}}
	\makebox[\textwidth][c]{\includegraphics[width=1.2\textwidth]{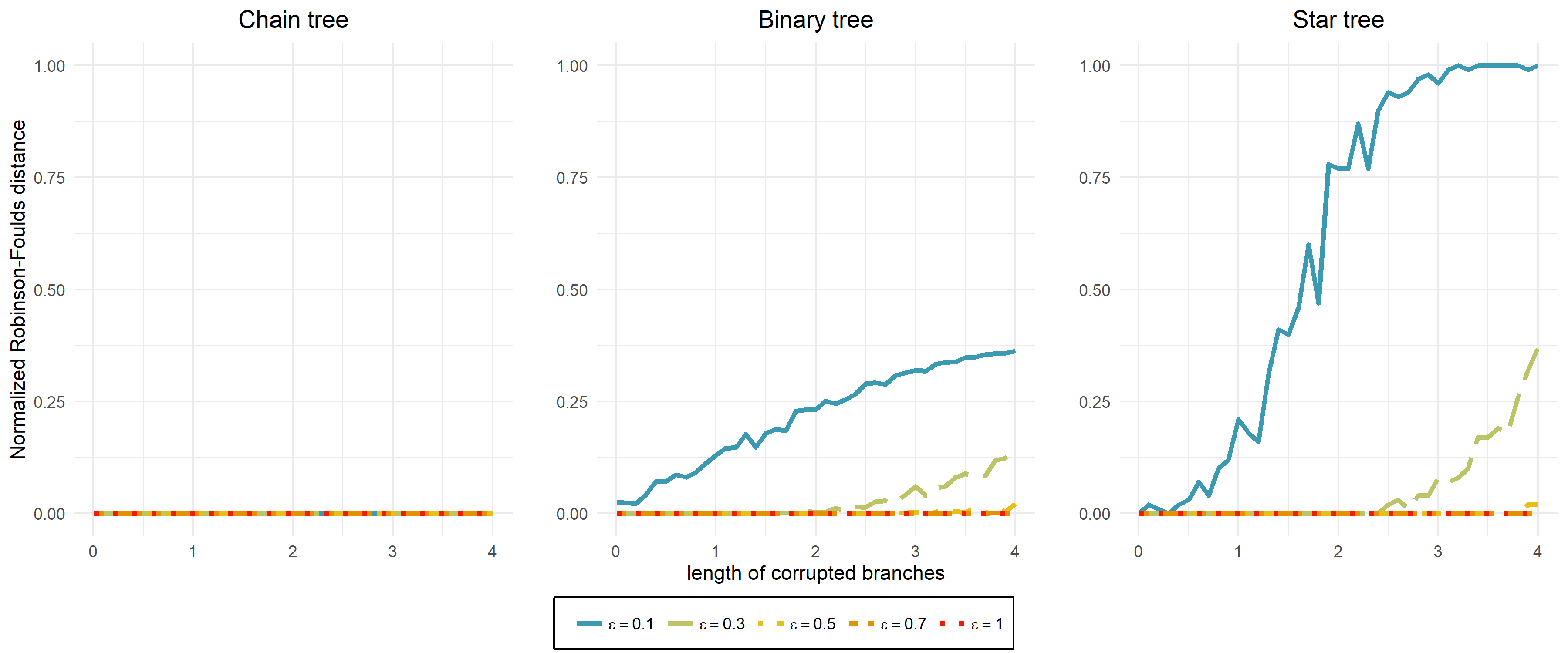}}%
	
	\caption{\label{fig:results_r24}Performance of NJ on corrupted data on $r=2$ states (top) and $r=4$ states (bottom) simulated on the trees of Fig \ref{fig:trees_simul}. The lengths of the edges of $T^*$ have been set to $1$ for $r=2$ and 2 for $r=4$ and the length $\ell$ of the corrupted branches varies in the x-axis. The figures show the normalized Robinson-Foulds distance between the recovered tree and $\overline{T^e}$ for different thresholds $\varepsilon$ for which the internal branches are shrunk. On the bottom left figure, the distance is zero for all values of $\varepsilon$ and all lengths of corrupted branches.}
\end{figure}

In practice, the choice of the threshold $\varepsilon$ from data could be done in a data-driven manner. However, any such procedure would be better implemented if there were some previous knowledge on the noise level or on underlying tree. Indeed, the case of a chain tree is dramatically different than a star tree with  the former being the one that gives best results and the star tree the worst.

For the binary tree, we have also implemented a slightly different procedure for tree recovery that uses the prior knowledge that the underlying tree is binary: we shrink the shortest internal edges of the tree output by NJ until the internal structure (i.e. removing external edges) gives a binary tree. We present these results in Figure \ref{fig:results_binary}. In this case, we obtain excellent results for $r=4$ (in both cases, for $r=2$ and $r=4$, the results are similar to the ones obtained with $\varepsilon=0.5$ and $1$, respectively).

\begin{figure}
\makebox[\textwidth][c]{\includegraphics[width=1.1\textwidth]{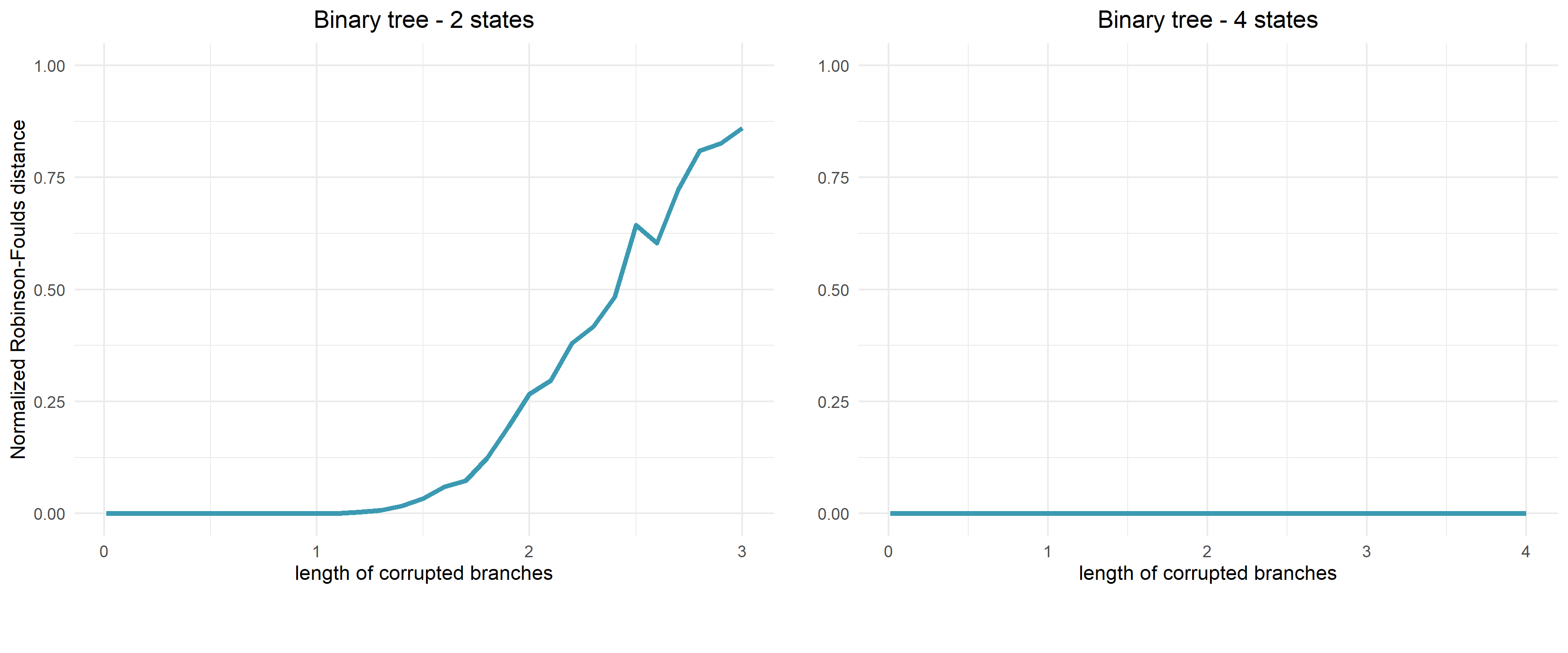}}
\caption{\label{fig:results_binary}Performance of NJ on corrupted data on $r=2$ states (left) and $r=4$ states (right) simulated on the binary tree of Fig \ref{fig:trees_simul}. The lengths of the edges of $T^*$ have been set to $1$ for $r=2$ and 2 for $r=2$ and the length $l$ of the corrupted branches varies in the x-axis. The figures show the normalized Robinson-Foulds distance between the recovered tree with its shortest edges shrunk until the internal structure is a binary tree and $\overline{T^e}$. }
\end{figure}

We also implemented the recovery of $T^*$ based on Theorem \ref{th:identifyexactly} for the same simulated data. If $\overline{T^e}$ is correctly obtained, then $T^*$ is successfully reconstructed most of the times (so we do not include figures with this information) for all the trees. Actually, for the binary and the star tree, $T^*$ is correctly recovered 100\% of the times (both for $r=2$ and $r=4$, and for any $\varepsilon$). For the chain tree the performance drops slightly for $r=2$ and $\ell>0.5$.  In this case 94\% of the times when $T^*$ is not correctly identified, the recovered tree differs from $T^*$ by one leaf.

%\marta{maybe we could include the figure only for the chain tree. Marina, try to look whether it fails for a single node or more.
%}
%

\section{Generalizations to linear models}\label{sec:linear}

We now briefly  mention a generalization to linear models on trees; see Section 2.3 in \cite{zwiernik2018latent} for more details. This generalizes the discrete case discussed in the previous sections, the Gaussian case, and some other cases of interest. In particular, it allows us to discuss continuous corruption models for discrete data as the one in Section~\ref{sec:conntinuous}.

\subsection{Linear models on trees}

In this section the vector $X$ takes values in any product space $\cX=\prod_{i=1}^d \cX_i$, where $\cX_i$ do not have to be discrete. We consider that $X$ follows a linear model on the tree $T^*$. This means that $X$ is Markov with respect to $T^*$ and for every edge $u\to v$ on the rooted version of $T^*$ the conditional expectation $\E(X_v|X_u)$ is an affine function of $X_u$. Models of this type were first discussed in~\cite{anandkumar2011spectral}.  A canonical example of such a situation is the Gaussian model on $T^*$. However, it also includes the discrete case discussed earlier. We explain this last connection in the following example.
\begin{ex}
If $X=(X_1,\ldots,X_d)$ is a discrete random vector with each component having $r$ values $\{0,1,\ldots,r-1\}$ we encode these states with $$\{0,e_1,\ldots,e_{r-1}\}\subset \R^{r-1},$$ where $e_i$ is the $i$-th canonical unit vector in $\R^{r-1}$ and $0\in \R^{r-1}$ is the zero vector. A ternary variable, for example, will take values $(0,0)$, $(1,0)$, $(0,1)$ in $\R^2$ instead of the typical $0,1,2$ in $\R$. The conditional expectation $$T(x_w)\;:=\;\E[X_u|X_w=x_w]$$ is an affine function of $x_w$. The value of $T(e_i)$ is simply the vector in $\R^{r-1}$ whose entries are the conditional probabilities of $X_u$ being $1,2,\ldots,r-1$ given $X_w=e_i$. Indeed, the conditional mean of this vector valued variable satisfies
$$
\E[X_u|X_w=e_i]\;=\;0\cdot \P(X_u=0|X_w=e_i)+\sum_{j=1}^{r-1}e_j\cdot \P(X_u=e_j|X_w=e_i).
$$
Similarly, $T(0)$ is the vector whose entries are the conditional probabilities of $X_u$ being $1,2,\ldots,r-1$ given $X_w=0$. We have $T(x_w)=Ax_w+b$ where $b=T(0)$ and the columns of $A$ are the vectors $T(e_i)-T(0)$.
%\marta{I'm a bit confused here: X is the joint distribution vector and takes $r$ values, the same  as the variables $X_u$ at the nodes of the tree? Or in this example the $X$ should be one of the $X_u$? }\piotr{$X=(X_1,\ldots,X_d)$. In our set-up each $X_i$ had $r$ possible values in $\{0,\ldots,r-1\}$ with some probabilities $\P(X_i=j)$. Formally, for each (univariate discrete variable) $X_i$ we introduce a variable $Z_i$ with values in $\{\mathbf 0,e_1,\ldots, e_{r-1}\}\subset \R^{r-1}$. $Z_i$ is a copy of $X_i$ in the sense that $\{X_i=j\}=\{Z_i=e_j\}$ for $j=1,\ldots,r-1$ (and so also $\{X_i=0\}=\{Z_i=\mathbf 0\}$). So they have the same distribution. Maybe using this notation with $Z$ would make things clearer?}\marta{No, it's clear now. I don't know why I got confused...maybe a subindex missing.}
\end{ex}

The fact that Gaussian undirected tree models fall into this category follows because the undirected graphical model over $T^*$ can be represented by a system of linear equations on the rooted version of $T^*$, where for each edge $u\to v$ we have $X_v=\lambda_{uv}X_u+\epsilon_v$ with $\lambda_{uv}\in \R$ and $\epsilon_{v}$ being zero-mean Gaussian and independent of each other. Then $\E[X_v|X_u]=\lambda_{uv}X_u$ is linear.

In general, let each variable $X_u$ for $u\in V$ be modelled as a random vector in $\R^{r-1}$ for a fixed $r$. Each variable can be either discrete or continuous but we add a requirement that:\\[-.1cm]
\begin{enumerate}
	\item [(AL1)] The matrix $\Sigma_{vv}=\E X_v X_v^T-\E X_v (\E X_v)^T$ is positive definite for every $v\in V$. \\[-.1cm]
	\end{enumerate}
This assumption has an analogous role as assumption (A1) for general Markov models.

To complete the model description we also assume that the corrupted version $X^e$ of $X$ depends in a linear way on $X$ in the sense that $\E(X_u^e|X_u)$ is an affine function of $X_u$ for every $u\in V$. As always, we assume that $X_i^e$ depends on $X$ only through the value of $X_i$. If $X$ follows a linear model on the tree $T^*$ then $(X,X^e)$ follows a linear model on the corresponding tree $T^e$. Our goal is to show that in this case we can recover the equivalence class $[T^*]$ from the distribution of $X^e$.

\subsection{The induced tree metric}

The distribution of $X^e$ gives a tree metric on $T^e$. Define the normalized version $\bar X_v$ of $X_v$ as $$\bar X_v:=(\Sigma_{vv})^{-1/2}(X_v-\E X_v).$$  Denoting $\Sigma_{uv}=\E X_u X_v^T-\E X_u (\E X_v)^T$ we obtain
\begin{equation}\label{eq:condexp}
	\E[\bar X_u|X_v]\;\;=\;\;\Sigma_{uu}^{-1/2}\Sigma_{uv}\Sigma_{vv}^{-1/2}\,\bar X_v,
\end{equation}
where we used that $\E[X_u|X_v]$ is affine, or equivalently, that $\E[X_u|X_v]=\E X_u+\Sigma_{uv}\Sigma_{vv}^{-1}(X_v-\E X_v)$. In analogy to (A2) we assume:\\[-.1cm]
\begin{enumerate}
	\item [(AL2)] For each edge $u\to v$ the matrix $\Sigma_{uv}$ is invertible and $\Sigma_{vv}\neq \Sigma_{vu}\Sigma_{uu}^{-1}\Sigma_{uv}$.\\[-.1cm]
\end{enumerate}
To see that Condition (AL2) is analogous to (A2) note that if $\Sigma_{vv}= \Sigma_{vu}\Sigma_{uu}^{-1}\Sigma_{uv}$ then $$A\;:=\;\Sigma_{uu}^{-1/2}\Sigma_{uv}\Sigma_{vv}^{-1/2}\;=\;(\Sigma_{vv}^{-1/2}\Sigma_{vu}\Sigma_{uu}^{-1/2})^{-1}.$$
	By (\ref{eq:condexp}) it follows that $\E[\bar X_u|\bar X_v]=A\bar X_v$ and $\E[\bar X_v|\bar X_u]=A^{-1}\bar X_u$. But this implies that $\E[\E(\bar X_u|\bar X_v)|\bar X_u]=\bar X_u$ and so $X_v$ contains all the information to fully recover $X_u$ and the other way around; meaning that these two variables are functionally related.

Define
$$\tau_{uv}:=\det(\Sigma_{uu}^{-1/2}\Sigma_{uv}\Sigma_{vv}^{-1/2})=\det (\E[\bar X_u \bar X_v^T]).$$

\begin{prop}\label{prop:taulinear}
	Under assumptions (AL1) and (AL2),  $\tau_{ij}^2\in (0,1)$ and  $d_{ij}=-\log\tau_{ij}^2$ defines a tree metric.
\end{prop}
\begin{proof}
We first show that $\tau_{uv}^2\in (0,1)$ for all edges $u\to v$. The fact that $\tau_{uv}$ cannot be zero follows immediately from (AL1) and (AL2). To show $\tau^2_{uv}<1$, equivalently we need to show that $\det\Sigma_{uv}^2< \det\Sigma_{uu}\det\Sigma_{vv}$. By applying Everitt's inequality \cite[Theorem 1]{everitt1958note} to the $2(r-1)\times 2(r-1)$ covariance matrix
$$
\Sigma\;=\;\begin{bmatrix}
	\Sigma_{uu} & \Sigma_{uv}\\
	\Sigma_{vu} & \Sigma_{vv}
\end{bmatrix}
$$
we conclude that $\det\Sigma_{uv}^2\leq  \det\Sigma_{uv}\det\Sigma_{vv}$ with equality if and only if  $\Sigma_{vv}=\Sigma_{vu}\Sigma_{uu}^{-1}\Sigma_{uv}$. By (AL2) this last condition cannot hold, proving that $\tau^2_{uv}<1$ for all edges $u\to v$.

We now show that $\tau_{ij}^2\in (0,1)$ for all $i,j$ and that the distances $d_{ij}=-\log\tau_{ij}^2$ form a tree matrix. Let $X_u,X_v,X_w$ be three random variables with values in $\R^{r-1}$ such that $X_u\indep X_w|X_v$. By the law of total expectation
$$
\E[\bar X_u \bar X_w^T]=\E\left[\E[\bar X_u|X_v](\E[ \bar X_w^T|X_v])^T\right]=\Sigma_{uu}^{-1/2}\Sigma_{uv}\Sigma_{vv}^{-1}\Sigma_{vw}\Sigma_{ww}^{-1/2},
$$
which implies that $\tau_{uw}=\det (\E[\bar X_u \bar X_w^T])=\tau_{uv}\tau_{vw}$. Applying this argument recursively we  conclude that the path-product decomposition of $\tau_{ij}$ given in~(\ref{eq:qij}) holds for any linear latent tree model; c.f. \cite{zwiernik2018latent} for more details. This implies that $\tau_{ij}^2\in (0,1)$ and that the collection of distances  $d_{ij}=-\log \tau_{ij}^2$ for all $i,j\in L$ gives a tree metric.
\end{proof}

Proposition~\ref{prop:taulinear} and Theorem~\ref{th:buneman} give immediatelly the following result.
\begin{thm}\label{th:mainlinear}
	If $(X,X^e)$ follows a linear model on $T^e$ and (AL1), (AL2) hold then $T^*$ can be recovered from the distribution of $X^e$ up to the equivalence class $[T^*]$.
\end{thm}

A special case of this set-up is given by the Gaussian model on $T^*$. We note that in the Gaussian case the Chow-Liu algorithm also boils down to computing the minimum weight spanning tree of the complete graph with weights $\hat d_{ij}=-\log \hat\rho^2_{ij}$. This shows that Theorem~\ref{th:CL} extends to this case.
\begin{prop}
	Theorem~\ref{th:CL} holds also for the Gaussian model on $T^*$.
\end{prop}

\subsection{Continuous corruption Ising model}\label{sec:conntinuous}

The linear model framework not only generalizes the discrete results but also it greatly extends possible models of corruption for which identifiability can be assured. In this section we briefly discuss a simple model of continuous corruption for binary data. So suppose $\cX=\{0,1\}^d$ but $X^e$ is a continuous random variable with values in $[0,1]^d$. For example, in image analysis applications, $X$ could have values black/white with $X^e$ taking values on the grayscale.

Since every function of a binary variable $X_i$ is affine, $X_i^e$ could be an  arbitrary random variable whose definition depends on $X$ only through $X_i$. We call this a continuous corruption Ising model. Theorem~\ref{th:mainlinear} immediately gives the following result.

\begin{thm}
In the continuous corruption Ising model satisfying assumptions (AL1) and (AL2) we can identify $[T^*]$ from the correlation matrix of $X^e$. If the noise satisfies condition (A4), $T^*$ can be identified uniquely.	
\end{thm}

To conclude we provide some simulations  where the conditional distribution of $X^e_i$ given $X_i=k$ is ${\rm Beta}(\alpha^i_k,\beta^i_k)$, which is a natural and tractable choice for a distribution on $[0,1]$. This means that the density of $X_i^e$ given $X_i=k\in \{0,1\}$ is
$$
p_i(y|k)\;=\;\frac{\Gamma(\alpha^i_k+\beta^i_k)}{\Gamma(\alpha^i_k)\Gamma(\beta^i_k)}y^{\alpha^i_k}(1-y)^{\beta^i_k}\qquad y\in [0,1],
$$
where $\Gamma$ denotes the Gamma function. For identifiability purposes we assume
\begin{equation}\label{eq:assumptionAB}
	\frac{\alpha_0^i}{\beta_0^i}\;<\;1\;<\;\frac{\alpha_1^i}{\beta_1^i}.
\end{equation}

For any edge $u\to v$ in $T^*$ the conditional expectation $\E[X_v|X_u]$ is a linear function of $X_u$. Similarly,
$$
\E[X_i^e|X_i=k]\;=\;(1-k)\frac{\alpha_0^i}{\alpha_0^i+\beta_0^i}+k\frac{\alpha_1^i}{\alpha_1^i+\beta_1^i}
$$
and so we obtain a version of a linear tree model. Assumption~(\ref{eq:assumptionAB}) assures that $\E[X_i^e|X_i=0]<\tfrac{1}{2}$ and $\E[X_i^e|X_i=1]>\tfrac{1}{2}$.

Given a distribution $X$ on a tree $T^*$, we have computed the corrupted distribution $X^e$ assuming $\alpha_0^i=\beta_1^i=1$ and $\alpha_1^i=\beta_0^i=a$, where the parameter $a$ varies from $2$ (that corresponds to the length of the corrupted edge $\ell=1.2$) to $5$ ($\ell=0.16$) satisfying the condition \ref{eq:assumptionAB}. For each possible value of $\alpha_1^i=\beta_0^i=a$ we have generated $1000$ samples of size $1000$ of the vector $X^e$.

The NJ algorithm has been used to recover the tree $\overline{T^e}$ from the sample correlation matrix of $X^e$. We measure how far is the recovered tree from $\overline{T^e}$ using the normalized Robinson-Foulds distance introduced in Section \ref{sec:simulDiscrete}. In Figure \ref{fig:results_cont} we present the mean of the normalized Robinson-Foulds distance for the  samples on the three trees of Figure \ref{fig:trees_simul}. As in Section \ref{sec:simulDiscrete} we present the results for different tolerances $\varepsilon$. This tolerance $\varepsilon$ is set such that we shrink the internal edges of the tree produced by NJ if the estimated length of the edge is smaller than $\varepsilon$. We can observe that in this case, the performance is higher than for the discrete case.

Similarly as in Section~\ref{sec:identify}, we have also studies the problem of recovering a binary tree using the prior knowledge that $T^*$ is binary. In this case we get almost 100\% recovery rate as soon as $a\geq 3$.

\begin{figure}
	\centering
	\makebox[\textwidth][c]{\includegraphics[width=1.2\textwidth]{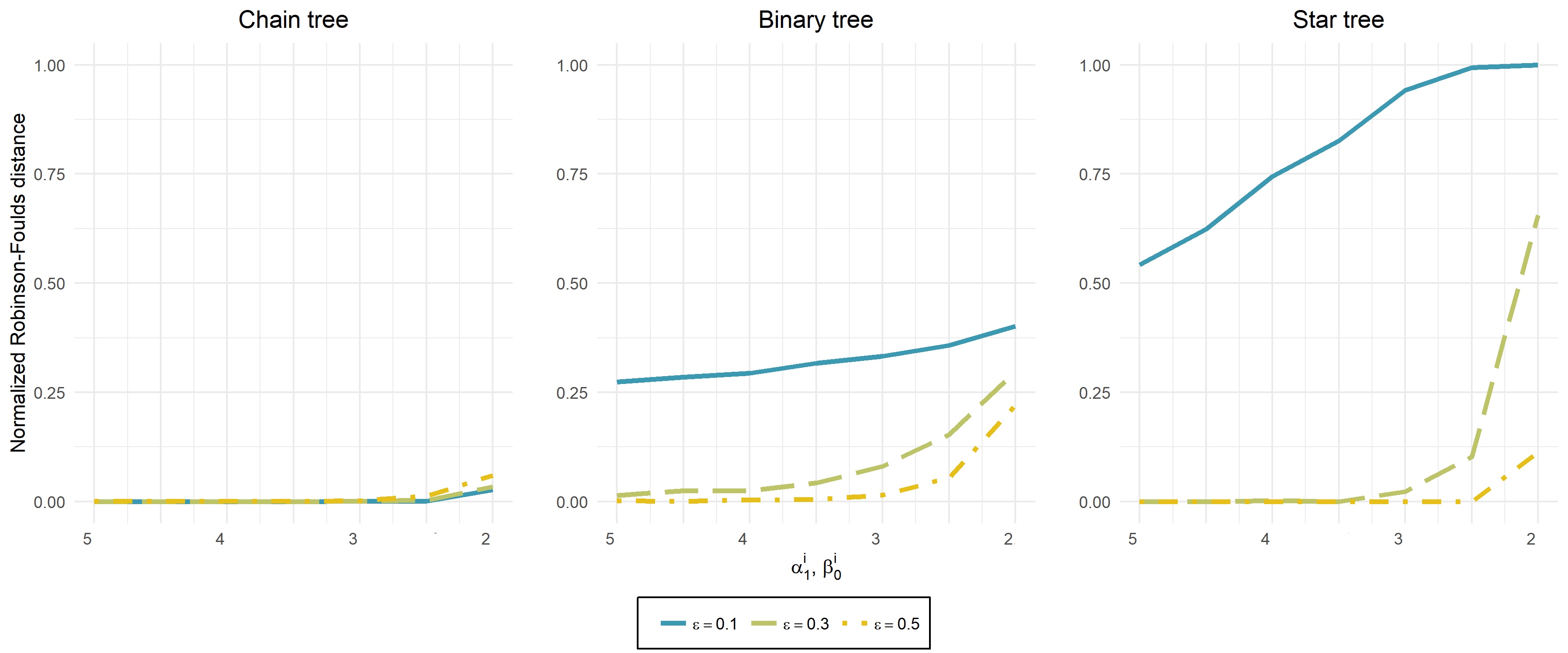}}
	
	\caption{\label{fig:results_cont}Performance of NJ on corrupted data on $r=2$ states simulated on the trees of Fig \ref{fig:trees_simul}. The lengths of the edges of $T^*$ have been set to $1$. Parameters $\alpha_0^i$ and $\beta_1^i$ are equal to $1$ while  $\alpha_1^i=\beta_0^i$ vary in the $x$-axis. The figures show the normalized Robinson-Foulds distance between the recovered tree and $\overline{T^e}$ for different thresholds $\varepsilon$ for which the internal branches are shrunk.}
\end{figure}

%The results presented in Figure \ref{fig:results_binary_cont} correspond to the trees recovered by shrinking the shortest internal edges of the tree output by NJ until the internal structure is a binary tree.

%\begin{figure}
%	%\includegraphics[scale=0.6]{RFdist_binary.png}
%	\makebox[\textwidth][c]{\includegraphics[width=0.7\textwidth]{cont_RF_binary_n1000.jpg}}
%	\caption{\label{fig:results_binary_cont}Performance of NJ on corrupted data on $r=2$ states simulated on the binary tree of Fig \ref{fig:trees_simul}. The lengths of the edges of $T^*$ have been set to $1$. Parameters $\alpha_0^i$ and $\beta_1^i$ are equal to $1$ while  $\alpha_1^i$ and $\beta_0^i$ vary in the $x$-axis. The figures show the normalized Robinson-Foulds distance between the recovered tree with its shortest edges shrunk until the internal structure is a binary tree and $\overline{T^e}$. }
%\end{figure}

We also implemented the recovery of $T^*$ based on Theorem \ref{th:identifyexactly} for the same simulated data. In the case of the binary and the star tree, $T^*$ is successfully reconstructed $100\%$ of the times if $\alpha_1^i=\beta_0^i>2$ and more than $97\%$ of the times for $\alpha_1^i,\beta_0^i=2$, independently of the chosen tolerance $\varepsilon$. The percentage of times that the correct tree $T^*$ is recovered for the case of the chain tree is presented in Figure \ref{fig:chain_Tstar}.

\begin{figure}
	\makebox[\textwidth][c]{\includegraphics[width=0.7\textwidth]{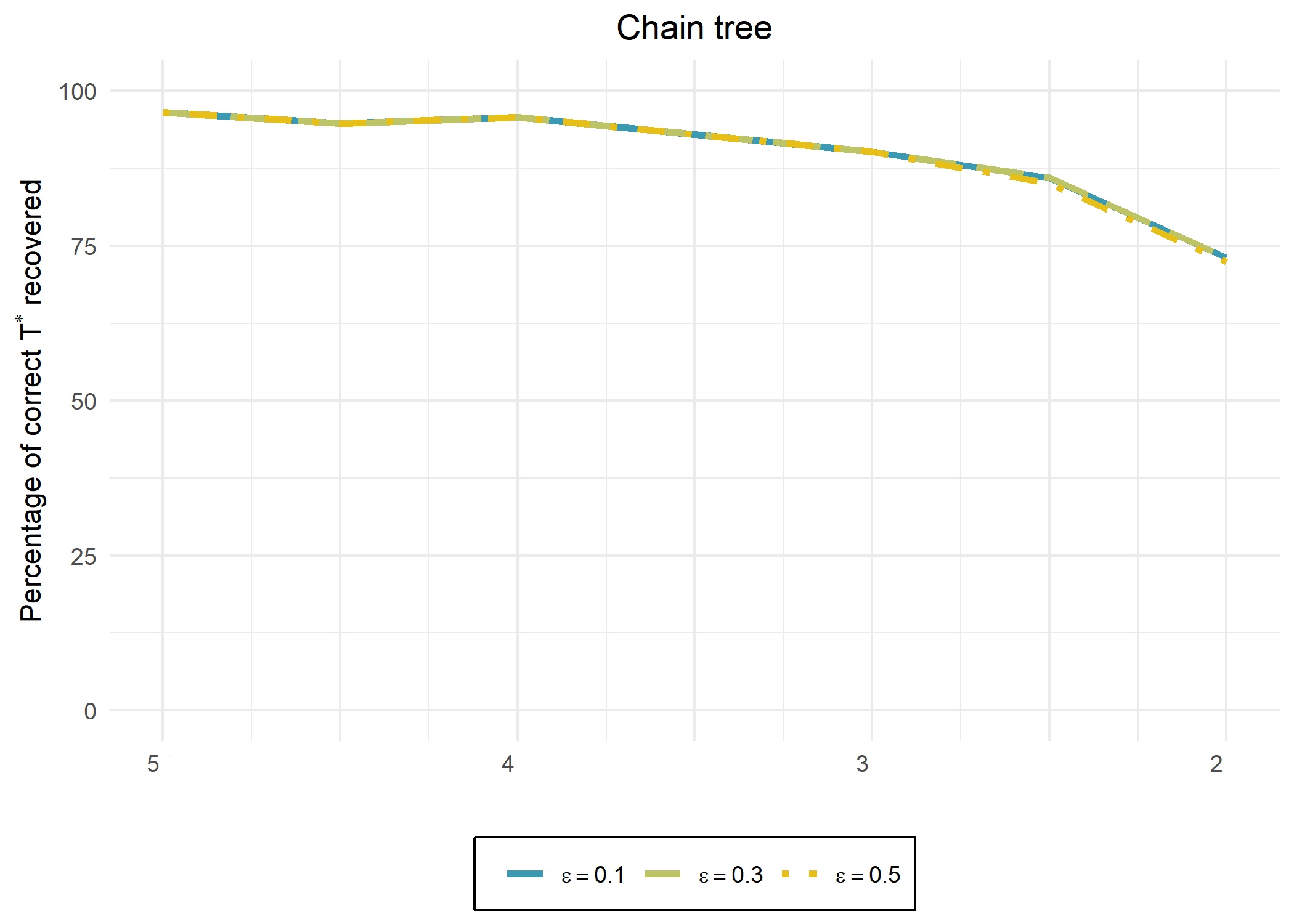}}
	\caption{\label{fig:chain_Tstar} Percentage of times that the chain tree $T^*$ is successfully reconstructed if $\overline{T^e}$ is correctly produced by NJ. Parameters $\alpha_0^i$ and $\beta_1^i$ are equal to $1$ while  $\alpha_1^i=\beta_0^i$ vary in the $x$-axis. The percentage of correct reconstructed trees $T^*$ are presented for different thresholds $\varepsilon$. }
\end{figure}

%We have
%$$
%{\rm cov}(X_i,X_i^e)=\left(\frac{\alpha_1^i}{\alpha_1^i+\beta_1^i}-\frac{\alpha_0^i}{\alpha_0^i+\beta_0^i}\right){\rm var}(X_i)
%$$
%
%By Theorem~\ref{} the underlying equivalence class $[T^*]$ can be recovered from the distribution of $X$ using only second order margins. To recover all the parameters of the model
%

\small
\section*{Acknowledgements}

PZ was supported from the Spanish Government grants (RYC-2017-22544, PGC2018-101643-B-I00), and Ayudas Fundaci\'on BBVA a Equipos de Investigaci\'on Cientifica 2017. MC and MGL were partially supported by Spanish Government Secretar\'ia de Estado de Investigaci\'on, Desarrollo e Innovaci\'on (MTM2015-69135-P MINECO/FEDER, PID2019-103849GB-I00 MINECO) and Generalitat de Catalunya (2014 SGR-634).

\bibliographystyle{plain}
\bibliography{bib_mtp2}

\end{document}